\documentclass[sigconf]{acmart}

\usepackage{multirow}
\usepackage{makecell}
\usepackage{microtype}
\usepackage{graphicx}
\usepackage{subfigure}
\usepackage{amsfonts}
\usepackage{amsmath,amsthm}
\usepackage{cleveref}
\usepackage{algorithmic}
\usepackage[ruled,linesnumbered]{algorithm2e}

\newcommand{\algo}{DEROG}

\AtBeginDocument{%
  }

\setcopyright{none}
\acmConference[WSDM '27]{The 20th ACM International Conference on Web Search and Data Mining}{February 15--19, 2027}{Hong Kong, China}
\acmYear{2027}
\authorsaddresses{}
\begin{document}

\title{Improving Graph Out-of-distribution Generalization Beyond Causality}

\author{Can Xu}
\email{leoxc1571@163.com}
\affiliation{%
  \institution{School of Data Science and Engineering, East China Normal University}
  \city{Shanghai}
  \country{China}
}

\author{Yao Cheng}
\email{yaocheng\_623@stu.ecnu.edu.cn}
\affiliation{%
  \institution{Nanyang Technological University}
  \country{Singapore}
}

\author{Jianxiang Yu}
\email{ffjasonyu@gmail.com}
\affiliation{%
  \institution{School of Data Science and Engineering, East China Normal University}
  \city{Shanghai}
  \country{China}
}

\author{Haosen Wang}
\email{haosenwang@seu.edu.cn}
\affiliation{%
  \institution{School of Cyber Science and Engineering, Southeast University}
  \city{Nanjing}
  \country{China}
}

\author{Jingsong Lv}
\email{jingsonglv@gmail.com}
\affiliation{%
  \institution{ZJU-Hangzhou Global Scientific and Technological Innovation Center, Zhejiang University}
  \city{Hangzhou}
  \country{China}
}

\author{Yao Liu}
\email{liuyao@cc.ecnu.edu.cn}

\author{Xiang Li}
\correspondingauthor
\email{xiangli@dase.ecnu.edu.cn}
\affiliation{%
  \institution{School of Data Science and Engineering, East China Normal University}
  \city{Shanghai}
  \country{China}
}

\begin{abstract}
Existing causal methods for graph out-of-distribution (OOD) generalization primarily assume the causal relationships between invariant sub-graphs and labels, which neglect the unignorable role of environment. However, in complex real-world applications, the assumption is easily violated. To address the problem, this paper characterizes environments through both covariate and concept shifts, and formalizes two properties: labels can depend on environmental information, and the predictive relevance of graph rationales can vary across environments. Based on these analyses, a novel variational inference based method DEROG (Probability \textbf{D}ependency on \textbf{E}nvironments and \textbf{R}ationales for \textbf{O}OD \textbf{G}raphs) is introduced. To alleviate the adverse effect of unknown prior knowledge on environments and rationales, \algo\ utilizes generalized Bayesian inference. Further, \algo\ employs an EM-based algorithm for optimization. Finally, extensive experiments on benchmark datasets under different distribution shifts are conducted to show the superiority of \algo.
\end{abstract}

\ccsdesc[500]{Mathematics of computing~Graph algorithms}
\ccsdesc[300]{Applied computing~Chemistry}
\keywords{Graph Neural Networks, OOD generalization, Variational inference}

\maketitle

\section{Introduction}
\label{sec:introduction}

Generalizing graph neural networks (GNNs) to out-of-distribution (OOD) graphs has recently received increasing attention. Most existing studies~\cite{isgib_23_yang,stablegnn_24_fan} tackle the problem from the perspective of causal inference and invariant learning, which aim at dissecting graphs into two non-overlapping parts, i.e., \emph{causal sub-graphs} and \emph{spurious sub-graphs}. 
Despite the difference in generating spurious sub-graphs, previous methods 
\cite{ciga_22_chen,dir_22_wu,dive_24_liang} generally assume that causal sub-graphs are invariant under distribution shifts and are used to determine properties / labels of graphs, while spurious sub-graphs characterize different distribution shifts and exert marginal impact on graph label prediction. 
However, in practice, the relationships among causal sub-graphs (graph rationales\footnote{In causal methods, they use causal sub-graphs to denote the graph rationales. Since our proposed method is not causal-based, we directly use graph rationales for distinction.}), environments, and labels are intricate. Therefore, existing methods could encounter major challenges when applied in complex scenarios. 

A crucial real-world application for graph OOD generalization is molecular learning. 
To illustrate, Fig.~\ref{fig:drug_bind} shows an example of the process for structurally specific drugs to take effect. At the stage of absorption and distribution, for drugs with the same pharmacophore (i.e., Barbiturates) as the sub-structure with the maximum efficacy information, both $\mathrm{R}^{1}$ and $\mathrm{R}^{2}$ are motifs that are considered superfluous according to previous studies \cite{ciga_22_chen}. In fact, $\mathrm{R}^{1}(\mathrm{R}^{2})$ determines whether drugs are able to reach target receptors and serves as the antecedent of valid drugs. This indicates that instead of being independent, properties of graphs are dependent on environments at certain level in real-world scenarios. At the metabolism stage, once drugs with the same pharmacophore successfully reach target receptors, factors like experimental assays, electron density distributions, geometry conformations, and metabolites also influence the validity of drugs. This demonstrates that causal sub-graphs could fail to exhibit rigorous invariance across different environments in predicting properties of graphs.


\begin{figure*}[tb]
    \centerline{\includegraphics[width=0.9\linewidth]{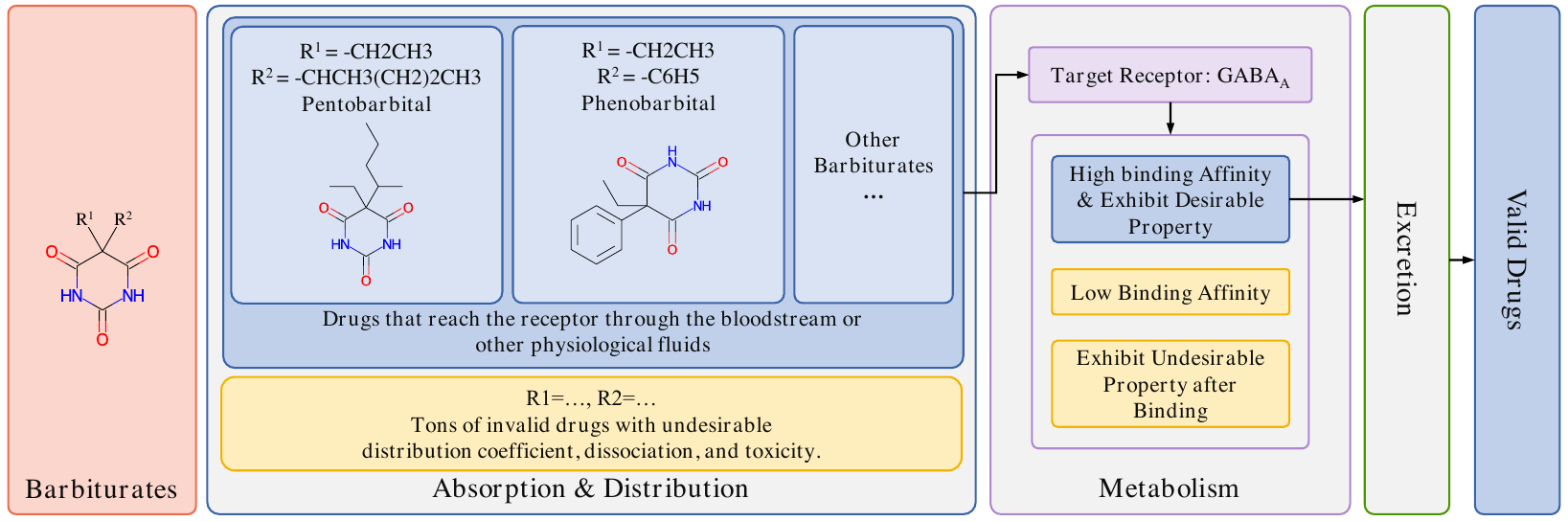}}
    \caption{The process of drugs to take effect 
    through binding with receptors. 
    This figure demonstrates 
    that there are a number of factors other than pharmacophore that determine validity and properties of drugs in real-world scenarios.}
    \Description{Diagram of barbiturate drugs progressing through absorption and distribution, receptor binding, metabolism, excretion, and the resulting drug validity.}
    \label{fig:drug_bind}
\end{figure*}

In this paper, we delve into graph OOD generalization beyond causality and 
propose a probabilistic graphical model \algo, which \textbf{D}epends on both \textbf{E}nvironments and \textbf{R}ationales to predict labels for \textbf{O}OD \textbf{G}raphs. Specifically, we model environments and graph rationales as hidden variables to capture the shift and consistency of data distributions, respectively. Different from previous causal methods that assume the independence of labels on environments and the rigorous invariance of rationales, \algo\ characterizes the complex relationships among environments, rationales, and graph labels from three aspects. First, the distribution of environments is characterized as the non-linear transformation of the joint distribution of graphs and labels. Second, due to the loose invariance of causal sub-graphs across various environments, we construct rationales based on both graphs and environments. Third, since environments can also determine labels, we jointly utilize environments, rationales, and graphs to predict overall properties.

Technically, \algo\ adopts an Expectation-Maximization (EM) based optimization framework. In the E-step, the parameters w.r.t. environments and rationales are optimized, while in the M-step, environments and rationales are fixed and used to improve the final classifier. Due to the limited knowledge regarding the prior distributions of environments and rationales, we leverage the generalized Bayesian inference technique, which substitutes the KL-divergence terms in the objective function with negative entropy. Further, we upgrade \algo\ by adding an environment alignment loss and a node-level contrastive loss, respectively. For the former, each graph could have a corresponding environment label that can be used to generate high-quality environment embeddings. For the latter, we sample positive and negative nodes based on the influence they exert in the final classifier, and the optimization of the contrastive loss leads to more expressive graph rationales. Finally, we summarize our main contributions as follows:

    \noindent $\bullet$
    We perform systematic analyses on the relationships among environments, graph rationales, and labels in the context of graph OOD generalization. 
    We identify two empirical patterns in real-world data: \emph{environment--label dependence} and \emph{environment-dependent rationale relevance}. 
    
    \noindent $\bullet$
    We present a novel probabilistic graphical model that models both environmental information and graph rationales.
    We leverage generalized Bayesian inference for hidden variables and put forward an EM-based framework for optimization.
    
    \noindent $\bullet$
    We conduct extensive experiments to evaluate the efficacy of \algo\ on various real-world benchmark datasets under different distribution shifts. Our results show the significant advantages of \algo\ over other competitors.

\section{Related Work}
\label{sec:relatedwork}
The quest for a more robust and effective deep neural network leads to wide research interest on OOD generalization. Common methods for OOD generalization include robust optimization~\cite{vrex_21_krueger,cnc_22_zhang}, domain adaption~\cite{dann_16_ganin,coral_16_sun,domaingen_23_wang} and invariant learning~\cite{causal_16_peters,irm_20_arjovsky,eiil_21_creager}. 

For the graph OOD problem~\cite{oodsurvey_25_li}, a number of causal inference methods~\cite{lisa_23_yu,dir_22_wu,gsat_22_miao,gsplice_24_li} aim at extracting invariant parts of graphs and then maximizing the correlations between these sub-structures and labels. For example, LECI~\cite{leci_23_gui} makes thorough exploration on the independence relation among invariant sub-graphs, environments and labels. Instead of focusing on extracting causal sub-graphs, some studies explore the roles of environments and spurious sub-graphs in 
invariant learning. GIL~\cite{gil_22_li} addresses the importance of inferring environment labels, which provide prior assumptions on distributions. GALA~\cite{gala_23_chen} addresses two pitfalls in causal learning and combines CIGA~\cite{ciga_22_chen} with pre-trained ERM to generate faithful environments. On node-level tasks, CaNet~\cite{canet_24_wu} points out that relations between ego-graphs and labels exhibit instability caused by confounding bias of latent environments. Therefore, a pseudo environment estimator is introduced to mitigate the gap.
MoleOOD~\cite{moleood_22_yang} 
is a recent variational inference based approach.
However,
it does not take the inter-dependency between variables into account when inferring environments. Also, it assumes consistent invariance of causal substructure across environments.

Further,
there are also 
methods that tackle OOD generalization by data augmentation. 
For example,
AIA~\cite{aia_23_sui} improves performance under covariate shift by augmenting input graphs with various spurious structures. IGM~\cite{igm_24_jia} combines invariant sub-graph with variant patterns to reduce the reliance on accurate environment labels. From a data-centric perspective, G-Splice~\cite{gsplice_24_li} proposes an environment-aware OOD graph generation method with non-Euclidean-space linear extrapolation designed in both structual and feature space.


\section{Preliminary}
\label{sec:preliminary}

\subsection{Problem Formulation}
\label{subsec:prob_form}
Let $\mathcal{Y}$, $\mathcal{G}$, and $\mathcal{E}$ denote graph label set, graph set, and environment set, respectively. The problem of graph OOD generalization aims to predict labels $Y \in \mathcal{Y}$ of corresponding graphs $G \in \mathcal{G}$ collected from various environments $E \in \mathcal{E}$. Distribution shifts on {$\mathrm{P}(G, Y) = \mathrm{P}(G)\mathrm{P}(Y|G)$} can be categorized into \emph{covariate shift} and \emph{concept shift}, depending on whether the shift occurs on $\mathrm{P}(G)$ or $\mathrm{P}(Y|G)$. Since shifts on the joint distribution $\mathrm{P}(G, Y)$ exist between training and test set, an optimal GNN classifier $\Phi(G)$ is supposed to predict labels with the minimum amount of risks, which is represented as $\Phi(*) = \arg\min \mathbb{E}_{\mathcal{G}, \mathcal{Y}}[l(\Phi(G), Y)]$.

\begin{table}[t]
    \centering
    \caption{Statistics of Edge Mask Weights in LECI}
    \label{tab:edge_weight_leci}
    \resizebox{0.45\textwidth}{!}{
        \begin{tabular}{lcccc}
        \toprule
        Dataset & \multicolumn{2}{c}{HIV} & LBAP-core-IC50 &Twitter \\
        \cmidrule(r){1-1} \cmidrule(l){2-5}
        Domain & Scaffold & Size & Assay &Length \\
        \cmidrule(r){1-1} \cmidrule(l){2-5}
        Shift & Covariate & Concept & Covariate & Concept \\
        \midrule
        Min. & 0.7068 & 0.8392 & 0.9042 & 0.9031 \\
        Max. & 0.7068 & 0.8874 & 0.9768 & 0.9957 \\
        Mean & 0.7068 & 0.8443 & 0.9434 & 0.9822 \\
        \bottomrule
        \end{tabular}
    }
\end{table}

\subsection{Preliminary Study}
In Section~\ref{sec:introduction}, we claimed that causal sub-graphs cannot independently determine properties of graphs under distribution shifts by a toy example in computational chemistry. 
We next further substantiate the following two claims with preliminary experimental results. 
First, causal sub-graph extraction is not a sufficient condition for desired OOD generalization performance. 
Second, 
the capability of SOTA causal learning models 
in extracting causal sub-graphs
does not meet the expectations of researchers when applied to complex real-world datasets. 

As the current best performer in graph OOD domain, LECI~\cite{leci_23_gui} extracts causal sub-graphs through the following process. 
First, it uses a GNN to obtain node representations. Then, it combines node representations from both ends of an edge and aggregates them through a Multi-Layer Perceptron (MLP) to derive a one-dimensional edge score. Subsequently, 
such edge score is passed through a Gumbel-Sigmoid activation function to obtain a mask weight that is close to 0 or 1, which can be used to well separate causal sub-graphs from spurious ones.
Ideally, 
for edges that connect nodes in causal sub-graphs, 
their weights should be close to 1, while that of other edges in the graph 
should be close to 0.
Finally,
these edge mask weights are used to control the strength of message propagation in subsequent GNN classifiers. 

We next 
inspect edge mask weights derived by LECI under different distribution shifts on three real-world graph OOD benchmark datasets.
Details on these distribution shifts and datasets will be given in Sec.~\ref{sec:experiments}.
Table~\ref{tab:edge_weight_leci}
summarizes the \emph{min}, \emph{max} and \emph{mean} values of mask weights w.r.t. all the edges for graphs in the test set.
From the table, we see that LECI assigns relatively large weights to all edges and the variance is small in all cases.
Notably, in the HIV-Scaffold-Covariate experiment, LECI assigns identical weights to 166,694 edges across 4,108 graphs from the test set. 
These results show that all the edges are considered important and the whole graphs are considered as causal sub-graphs.
However, 
\textbf{this phenomenon contradicts with 
the presence of distribution shifts}.
Further,
LECI can still achieve the best performance among all causal methods (see results in Table~\ref{tab:real-world_results} later).

To further illustrate the 
environment-label dependency in OOD scenarios 
quantitatively, 
we conduct additional experiments using 
causal-based baselines. Since LECI does not perform sub-graph partitioning, we use CIGA, which divides graphs into causal and spurious parts, and are then respectively used for label prediction and yield $c_{out}$ and $s_{out}$. 
We compute the loss between $c_{out}$ and $y$, and
that between $s_{out}$ and $y$ on the test set, 
whose results are shown in Table~\ref{tab:csout_ciga}. We observe that 
the spurious graph, which has been previously considered to represent the environment distribution, achieves similar levels of accuracy in label prediction as the 
causal sub-graph. 
This 
further suggests that 
the 
causal and spurious sub-graphs identified by existing methods on real datasets, in fact, 
exhibit comparable levels of mutual information with the labels.

All
these observations show that 
superior OOD generalization performance does not necessarily require causal sub-graph extraction. 
Additionally, existing causal learning methods are less capable in effectively extracting causal sub-graphs that determine graph properties.
\emph{These observations motivate us to study non-causal methods that do not distinguish causal and spurious sub-graphs}.

\subsection{Environments and Rationales}
\label{subsec:emp_expl}

\begin{table}[t]
    \centering
    \caption{Preliminary studies on $c_{out}$ and $s_{out}$ from CIGA}
    \label{tab:csout_ciga}
        \begin{tabular}{lccc}
        \toprule
        Dataset & \multicolumn{2}{c}{HIV} & Twitter \\
        \cmidrule(r){1-1} \cmidrule(l){2-4}
        Domain & Scaffold & Size  &Length \\
        \cmidrule(r){1-1} \cmidrule(l){2-4}
        Shift & Covariate & Concept & Covariate \\
        \midrule
        $loss(c_{out}, y)$ & 0.1035 & 0.2218 & 2.5274 \\
        $loss(s_{out}, y)$ & 0.1035 & 0.2220 & 2.4446 \\
        
        \bottomrule
        \end{tabular}
\end{table}
Previous invariant learning methods encounter two major pitfalls: one is the idealized assumption of independence between environments and labels, 
while the other is the excessive emphasis on the invariance of causal sub-graphs. 
These assumptions can be easily satisfied on most synthetic datasets, where certain sub-graphs determine the overall properties and remaining parts of graphs are completely superfluous. 

However, real-world cases are more complex where these assumptions are more likely to fail (See Fig.~\ref{fig:drug_bind}).
Therefore, we view an environment $E$ as a latent descriptor of the data-generating distribution. Under covariate shift, $\mathrm{P}_{\mathrm{tr}}(G) \neq \mathrm{P}_{\mathrm{te}}(G)$ while $\mathrm{P}_{\mathrm{tr}}(Y\mid G) = \mathrm{P}_{\mathrm{te}}(Y\mid G)$; under concept shift, the conditional distributions differ. Since $\mathrm{P}(G,Y)=\mathrm{P}(G)\mathrm{P}(Y\mid G)$ changes in either case, $E$ should encode graph- and label-related variation. This motivates conditioning its inference on $G$ and an inferred pseudo-label $\tilde{Y}$.

\begin{proposition}
\label[proposition]{prop:env_g_y}
Under covariate shift, the environment-dependent variation is characterized by $\mathrm{P}(G)$; under concept shift, it is characterized by $\mathrm{P}(Y\mid G)$. Hence, an environment variable for graph OOD generalization must encode variation in the joint distribution $\mathrm{P}(G,Y)$.
\end{proposition}

The widely used real-world datasets for graph OOD generalization are molecule and NLP-related datasets~\cite{good_22_gui,drugood_23_ji}. For the former, Fig.~\ref{fig:drug_bind} illustrates that pharmacophores play a decisive role in receptor binding, while efficacy for molecules with the same pharmacophoric group also depends on other factors. For the latter, text sequences are converted into graphs where nodes represent words and edges represent correlations; irony, paradox, and context can similarly alter labels beyond an invariant sub-graph. These cases motivate the following two properties.

\begin{theorem}
\label{theo:dep_on_env}
\textbf{Environment--label dependency.} Let $(G_i,Y_i)$ and $(G_j,Y_j)$ be graph--label pairs from the training and test environments, respectively, and let $\mathrm{Q}_{\Phi_{\hat{G}}}$ be a rationale extractor optimized on the training data. Even when $\Phi^*_{\hat{G}}=\arg\max I(\hat{G}_i;Y_i)$, where $\hat{G}_i=\mathrm{Q}_{\Phi_{\hat{G}}}(G_i)$, the complementary sub-graph in the shifted environment can retain label information: $I(G_j\setminus\hat{G}_j;Y_j)>0$.
\end{theorem}

\begin{theorem}
\label{theo:var_inv}
\textbf{Environment-dependent rationale relevance.} Given a graph $G$ with rationale $\hat{G}$, there can exist a subset $\mathcal{E}_{sub}\subseteq\mathcal{E}$ and an environment $E_k\in\mathcal{E}\setminus\mathcal{E}_{sub}$ such that, for any $E_i,E_j\in\mathcal{E}_{sub}$,
\[
\mathrm{P}(Y\mid\hat{G},E_i)=\mathrm{P}(Y\mid\hat{G},E_j)\neq\mathrm{P}(Y\mid\hat{G},E_k).
\]
\end{theorem}

Thus, a rationale can provide stable signal within similar environments while its predictive relevance changes under a sufficiently different shift. These results motivate inference that retains both environmental information and graph rationales, without imposing environment--label independence or universal rationale invariance. Given the complex nature of distribution shifts in real-world scenarios, our method makes a relaxed inference on environments and graph rationales.

\section{Methodology}
\label{sec:methodology}


\subsection{The Variational Framework}
\label{subsec:vi_goodg}

\begin{figure*}[tb]
    \centerline{\includegraphics[width=0.9\linewidth]{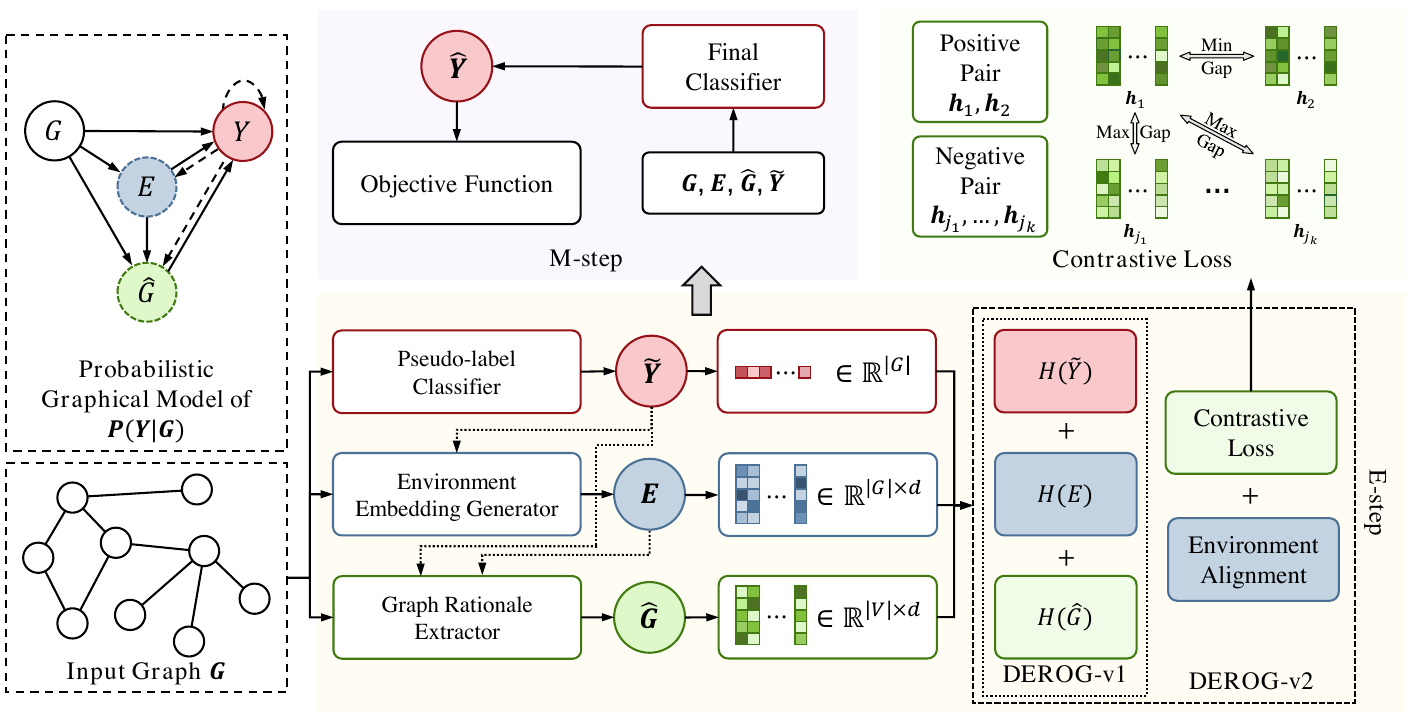}}
    \caption{The overall framework of the proposed method \algo. 
    }
    \Description{Framework diagram of DEROG, showing its probabilistic model, latent environment and rationale inference, classifier optimization, and contrastive learning components.}
    \label{fig:overall_framework}
\end{figure*}

Motivated by the analysis in Section~\ref{subsec:emp_expl}, we propose a probabilistic graphical model that captures environmental information and environment-dependent rationale relevance, as shown in the upper-left corner of Fig.~\ref{fig:overall_framework}. We rewrite $P(Y|G)$ as:
\begin{align}
    \mathrm{P}(Y|G) = \iiint &\mathrm{P}(Y | \hat{G}, E, \tilde{Y}, G)\mathrm{P}(\hat{G} | E, \tilde{Y}, G) \mathrm{P}(E | \tilde{Y}, G)  \nonumber \\
    &\mathrm{P}(\tilde{Y} | G) d\hat{G} dE d\tilde{Y},
\label{eq:pyg}
\end{align}
where $\hat{G}$, $E$, and $\tilde{Y}$ are hidden variables denoting graph rationales, environments, and pseudo-labels, respectively. 
As discussed in Section~\ref{subsec:emp_expl}, $E$ represents joint graph--label variation. Since $Y$ is to be predicted, we introduce pseudo-labels $\tilde{Y}$ and use $G$ and $\tilde{Y}$ to infer $E$ in Eq.~\ref{eq:pyg}.
Due to the intractability of latent variables regarding $P(Y|G)$, we cannot directly maximize its likelihood. Hence, we resort to the variational inference (VI) technique and maximize the evidence lower bound (ELBO) of $P(Y|G)$.

\begin{proposition}
\label[proposition]{prop:elbo_gy}
    The ELBO of $\mathrm{P}(Y|G)$ takes the form of:
    \begin{align}
    \mathcal{L} = &\mathbb{E}_{\mathrm{Q}_{\Phi_{\hat{G}}}, \mathrm{Q}_{\Phi_{E}}, \mathrm{Q}_{\Phi_{\tilde{Y}}}} \left[ \log \mathrm{P}_{\theta_{Y}}(Y|\hat{G}, E, \tilde{Y}, G) \right] - \nonumber \\
    &\mathrm{KL}\left[\mathrm{Q}_{\Phi_{\hat{G}}}(\hat{G}|E, \tilde{Y}, G) || \mathrm{P}(\hat{G}|E, \tilde{Y}, G)\right] - \nonumber\\
    &\mathrm{KL}\left[\mathrm{Q}_{\Phi_{E}}(E|\tilde{Y}, G) || \mathrm{P}(E|\tilde{Y}, G)\right] - \mathrm{KL}\left[\mathrm{Q}_{\Phi_{\tilde{Y}}}(\tilde{Y}|G) || \mathrm{P}(\tilde{Y}|G)\right],
    \label{form:elbo}
    \end{align}
    where $\mathrm{Q}_{\Phi_{\hat{G}}}$, $\mathrm{Q}_{\Phi_{E}}$, and $\mathrm{Q}_{\Phi_{\tilde{Y}}}$ are variational distributions parameterized by $\Phi_{\hat{G}}$, $\Phi_{E}$, and $\Phi_{\tilde{Y}}$, respectively.
    $\theta_Y$ also denotes learnable parameters of the final classifier.
\end{proposition}

\textbf{Pseudo-label classifier.} Since graph labels are not accessible during training, we introduce a pseudo-label classifier $\mathrm{Q}_{\Phi_{\tilde{Y}}}$. Specifically, given a graph $G=(A, X)$ with the adjacency matrix $A$ and the node feature matrix $X$, we first use a {GNN} to generate node embeddings and then employ a {Readout} function to derive the graph embedding. After that, with $c$ classes in total, the embedding is further mapped to the probability distribution of pseudo-labels with a MLP. Formally, we have:
\begin{align}
    \tilde{Y} = \mathrm{MLP}(\mathrm{Readout}(\mathrm{GNN}(X, A))) \in \mathbb{R}^{c}.
\end{align}

\textbf{Environment embedding generator.} In our probabilistic graphical model, environments are derived based on both graphs and labels. Therefore, given a graph $G$ with $n$ nodes, to generate the environment embedding, we first enrich node features by concatenating $X$ with the pseudo-label of graph, and then fed into a GNN encoder. Note that since $\tilde{Y} \in \mathbb{R}^c$ is a graph-level feature, before concatenation, we need to expand the dimension of $\tilde{Y}$ and replicate it $n$ times by row. After $L$ layers, we further add an additional gradient reverse layer (GRL)~\cite{grl_13_ganin} to promote the adversarial training of environments. The overall procedure is given as follows:
\begin{gather}
    \mathbf{H}_{E}^{L} = \mathrm{GRL}(\mathrm{GNN}([X, \mathrm{Linear}(\tilde{Y})], A)) \in \mathbb{R}^{n \times d}, \\
    E = \mathrm{Readout}(\mathbf{H}_{E}^{L}) \in \mathbb{R}^{d},
\end{gather}
where $d$ is the dimensionality of hidden embeddings and $\textrm{Linear}(\cdot)$ represents the linear transformation.

\textbf{Graph rationale extractor.} Once the environment of $G$ is obtained, the graph rationale can be encoded similarly:
\begin{equation}
    \hat{G} = \mathrm{\sigma}(\mathrm{GNN}([X, E, \mathrm{Linear}(\tilde{Y})], A)) \in \mathbb{R}^{n \times d},
\label{eq:rationale}
\end{equation}
where $\sigma$ is the sigmoid function. Similarly, since $E$ and $\tilde{Y}$ are graph-level embeddings, we need to first expand and replicate them before concatenation with $X$. Different from the environment, we represent the graph rationale in the form of node-level representation. This is because nodes contribute differently to the rationale and node-level representation can retain the importance of each node to the rationale.

\textbf{Final classifier.} To implement the final classifier $\mathrm{P}_{\theta_{Y}}$, we jointly leverage environments, graph rationales, and pseudo-labels.
We first use a $L$-layer GNN encoder to generate node embeddings $\mathbf{H}_{Y}^{L}$. Then, $\mathbf{H}_{Y}^{L}$ is further combined with $\hat{G}$ to memorize rationale information. Finally, we take $\mathbf{H}_{Y}^{L}$ and $\hat{G}$ as input, and use an MLP to predict graph labels. Specifically, we have:
\begin{gather}
    \mathbf{H}_{Y}^{L} = \mathrm{GNN}([X, E, \mathrm{Linear}(\tilde{Y})], A) \in \mathbb{R}^{n \times d}, \\ 
    \hat{Y} = \mathrm{MLP}(\mathrm{Readout}(\mathbf{H}_Y^L \odot \hat{G})) \in \mathbb{R}^{c}. 
\end{gather}


\subsection{E-step: Inference on Environments, Graph Rationales, and Pseudo-labels}
\label{subsec:estep}
The VI-based graph OOD generalization problem has been formulated as maximizing the ELBO of $\mathrm{P}(Y|G)$, which is given in Eq.~\ref{form:elbo}. In the E-step, we aim to optimize the parameters regarding environments, rationales, and pseudo-labels, which reduces the objective function in Eq.~\ref{form:elbo} to:
\begin{align}
    \max \Big[&-\mathrm{KL}\left(\mathrm{Q}_{\Phi_{\hat{G}}}(\hat{G}|E, \tilde{Y}, G) || \mathrm{P}(\hat{G}|E, \tilde{Y}, G)\right) - \nonumber \\
    &\mathrm{KL}\left(\mathrm{Q}_{\Phi_{E}}(E|\tilde{Y}, G) || \mathrm{P}(E|\tilde{Y}, G)\right) - \nonumber \\ 
    &\mathrm{KL}\left(\mathrm{Q}_{\Phi_{\tilde{Y}}}(\tilde{Y}|G) || \mathrm{P}(\tilde{Y}|G)\right)\Big].
    \label{form:estep1}
\end{align}

However, the prior distributions of $\tilde{Y}$, $E$, and $\hat{G}$ are unknown. It is impracticable for us to induce a faithful prior distribution $\mathrm{P}(\cdot|\cdot)$ in a non-parametric manner like some node-level tasks~\cite{glem_23_zhao}. To address the problem, we adopt the core principle of generalized Bayesian inference~\cite{gbi_68_dempster} and substitute the KL divergence with a convex function, which takes the form of negative entropy in this study. Then we have:

\begin{proposition}
    \label[proposition]{prop:gen_baye_inf}
    The optimization objective in the E-step can be transformed into:
    \begin{align}
        \mathcal{L}_{\mathbf{E_{v1}}} = &-\lambda_{\Phi_{\hat{G}}} H(\mathrm{Q}_{\Phi_{\hat{G}}}(\hat{G}|E, \tilde{Y}, G)) - \lambda_{\Phi_{E}} H(\mathrm{Q}_{\Phi_{E}}(E|\tilde{Y}, G)) - \nonumber \\
        &\lambda_{\Phi_{\tilde{Y}}} H(\mathrm{Q}_{\Phi_{\tilde{Y}}}(\tilde{Y}|G)),
        \label{eq:entropy}
    \end{align}
    where $H(\cdot)$ is the entropy function.
\end{proposition}

Inspired by $\beta$-VAE~\cite{betavae_17_higgins}, we use $\lambda$ to balance the scale of each term. In our implementations, the ratio between $\lambda_{\Phi_{\hat{G}}}$, $\lambda_{\Phi_{E}}$, and $\lambda_{\Phi_{\tilde{Y}}}$ is fixed at 1:1:10. By directly optimizing the objective $\mathcal{L}_{\mathbf{E_{v1}}}$, we propose the first version of \algo, denoted as \algo-v1.

To further enrich the connotation of latent environments and enhance the salience of important nodes, we further introduce an environment alignment loss and a contrastive loss. For the former, each environment can be generally associated with an id, which serves as the label of distribution shift. Therefore, we adopt the cross-entropy loss to drive the learned environment embeddings to be easily separable. In line with our environmental modeling, this term enhances the model's capability in distinguishing various distribution shifts. Given a graph $G$ and its environment id $Y_{E}$, we derive:
\begin{align}
    {\mathcal{L}_{env}} = \mathrm{CrossEntropy} (Y_{E}, \mathrm{Linear}(E)).
\end{align}

On the other hand, it is expected that a well-learned rationale extractor should
learn more (less) information from key (irrelevant) nodes.
To achieve this goal, considering that the learned $\hat{G} \in \mathbb{R}^{n\times d}$ contains the importance of each node to the rationale, we add a contrastive loss, which pushes embeddings of key nodes away from others. 
Given a graph $G$, we consider that nodes that are more likely to contribute to the rationale are positive samples, while others are negative ones. Since the sigmoid function used to generate $\hat{G}$ ensures that each entry in $\hat{G}$ ranges from 0 to 1, we simply measure the importance score of each node $v_i$ by computing the $L_1$ norm of its embedding $\hat{g}_i$. Then nodes are sorted in descending order by their scores and divided into two halves. We next show how to construct positive and negative pairs as follows. We first randomly sample a node $v_1$ as anchor. Then we select a random node $v_2$ as positive sample from the same half where $v_1$ is located. After that, we randomly select $k$ nodes from the other half as negative samples. Finally, the contrastive loss is given by:
\begin{align}
    \mathcal{L}_{cl} = -  \log \frac{\exp (h_{1}^T h_{2} / \tau)}{\sum_{v_{j} \in \mathcal{V}_{-}} \exp (h_{1}^T h_{j} / \tau)},
\end{align}
where $\mathcal{V}_{-}$ represents the negative node set with $k$ samples. We add both $\mathcal{L}_{env}$ and $\mathcal{L}_{cl}$ to $\mathcal{L}_{\mathbf{E_{v1}}}$, and formulate a new loss function:
\begin{align}
    \mathcal{L}_{\mathbf{E_{v2}}} = \mathcal{L}_{\mathbf{E_{v1}}} + \lambda_{env} \mathcal{L}_{env} + \lambda_{cl} \mathcal{L}_{cl}.
    \label{eq:func_v2}
\end{align}
Optimizing $\mathcal{L}_{\mathbf{E_{v2}}}$ leads to the second version of \algo, denoted as \algo-v2.

\begin{table*}[t]
    \centering
    \caption{Performance comparisons in real-world scenarios. The evaluation metric employed for the two molecule datasets is ROC-AUC, whereas classification accuracy is adopted for NLP datasets. The best result is emphasized in bold while the second-best one is underscored.}
    \label{tab:real-world_results}
    \resizebox{1.00\textwidth}{!}{
    \begin{tabular}{lccccccccccc}
        \toprule
        Task & \multicolumn{7}{c}{Molecule Learning} & \multicolumn{4}{c}{Sentiment Analysis} \\
        \midrule
        Dataset & \multicolumn{4}{c}{GOODHIV} & \multicolumn{3}{c}{DrugOOD-LBAP-core-IC50} & \multicolumn{2}{c}{GOODTwitter} & \multicolumn{2}{c}{GOODSST2} \\
        \cmidrule(r){1-1} \cmidrule(lr){2-8} \cmidrule(l){9-12}
        Domain & \multicolumn{2}{c}{Scaffold} & \multicolumn{2}{c}{Size} & \multicolumn{1}{c}{Scaffold} & \multicolumn{1}{c}{Size} & \multicolumn{1}{c}{Assay} & \multicolumn{2}{c}{Length} & \multicolumn{2}{c}{Length} \\
        \cmidrule(r){1-1} \cmidrule(lr){2-5} \cmidrule(lr){6-8} \cmidrule(lr){9-10} \cmidrule(l){11-12}
        Shift & Covariate & Concept & Covariate & Concept & Covariate & Covariate & Covariate & Covariate & Concept & Covariate & Concept \\
        \midrule
        ERM & 72.01(1.02) & 67.23(2.41) & 63.05(2.80) & 61.88(4.27) & 70.03(0.59) & 67.02(0.21) & 69.93(0.40) & 57.77(0.75) & 45.44(1.85) & 78.94(0.55) & 71.35(0.56) \\ 
        IRM & 71.05(1.25) & 64.05(2.30) & 67.94(1.07) & 70.92(7.60) & 67.78(0.51) & 65.37(0.51) & 71.33(0.73) & 55.89(1.41) & 49.68(3.13) & 80.21(1.13) & 72.17(1.47) \\ 
        VREx & 67.78(1.80) & 65.68(3.42) & 64.29(1.98) & 57.85(1.45) & 68.52(0.46) & 65.09(1.50) & 71.06(0.93) & 57.01(0.08) & 49.93(3.13) & 78.62(1.83) & 71.59(0.20) \\ 
        DANN & 72.45(0.90) & 67.54(2.26) & 64.61(4.41) & 64.23(3.80) & 67.85(1.20) & 65.93(0.57) & 70.86(0.38) & 56.51(1.33) & 48.31(0.65) & 79.71(0.25) & 73.87(1.74) \\ 
        \midrule
        DIR & 70.72(0.73) & 68.59(1.78) & 60.28(2.81) & \underline{76.73(2.83)} & 68.23(0.58) & 63.53(1.41) & 71.28(1.46) & 55.94(1.00) & 49.39(1.39) & 80.02(1.16) & 62.46(2.90) \\ 
        GSAT & 70.96(1.36) & 68.23(1.41) & 63.78(1.36) & 63.94(7.00) & 69.14(0.44) & 66.73(0.36) & 71.94(0.76) & 57.14(0.65) & 48.07(1.78) & 80.78(1.98) & 71.76(0.25) \\ 
        CIGA & 69.19(1.01) & 69.89(0.57) & 61.78(0.45) & 75.90(1.28) & 67.47(1.12) & 64.62(0.90) & 71.35(0.50) & 57.49(1.04) & 49.79(3.85) & 79.67(2.16) & 68.70(1.88) \\ 
        GALA & 68.00(4.88) & \underline{70.93(4.17)} & 60.82(5.48) & 75.18(2.48) & 66.97(0.21) & 65.36(0.42) & 70.46(0.76) & 57.35(07.0) & 49.31(2.41) & 81.36(0.63) & 67.70(0.72) \\ 
        AIA & 67.56(1.96) & 65.18(1.10) & 61.64(3.37) & 69.90(1.37) & 65.50(0.76) & 63.53(0.12) & 70.89(0.81) & 56.20(0.86) & 46.66(1.92) & 78.80(2.13) & 69.43(2.32) \\ 
        iMoLD & 69.90(1.55) & 70.65(0.22) & 63.63(2.07) & 73.76(0.45) & 68.68(0.90) & 66.09(0.27) & 71.00(0.29) & 56.78(1.41) & 48.33(2.99) & 56.78(1.41) & 73.14(0.48) \\ 
        EQuAD & 67.38(1.47) & 63.85(0.84) & 60.86(1.70) & 48.52(2.01) & 64.48(0.75) & 62.93(0.64) & 69.75(0.92) & 57.79(1.25) & 43.81(1.63) & 75.40(1.92) & 67.06(1.38) \\ 
        IGM & 66.20(1.01) & 63.90(2.47) & 61.30(1.82) & 66.00(0.79) & 67.80(1.60) & \textbf{70.80(2.04)} & \textbf{81.60(2.53)} & 57.90(1.44) & 52.70(0.73) & 80.26(2.57) & 72.06(1.32) \\ 
        LECI & 72.50(0.85) & 69.48(1.39) & 64.86(2.40) & 55.91(2.76) & 69.09(0.58) & 66.50(0.40) & 72.62(0.26) & 59.64(0.60) & 53.62(2.66) & \textbf{83.46(0.33)} & 72.97(1.62) \\ 
        \midrule
        \algo-v1 & \underline{75.51(0.66)} & 70.45(1.53) & \underline{68.88(1.28)} & 76.33(1.99) & \underline{70.56(0.18)} & 67.88(0.51) & 72.78(0.22) & \underline{60.15(0.58)} & \textbf{54.90(0.79)} & 81.57(0.17) & \underline{73.58(0.56)} \\ 
        \algo-v2 & \textbf{76.87(0.80)} & \textbf{71.40(1.50)} & \textbf{70.48(0.30)} & \textbf{80.65(8.35)} & \textbf{70.87(0.21)} & \underline{68.07(0.75)} & \underline{73.03(0.21)} & \textbf{60.54(0.43)} & \underline{54.35(0.33)} & \underline{82.00(0.42)} & \textbf{75.48(1.17)} \\ 
        \bottomrule
    \end{tabular}}
\end{table*}

\subsection{M-step: Classifier Optimization}
\label{subsec:mstep}
In the M-step, we fix 
the learned parameters of $\mathrm{Q}_{\Phi_{\tilde{Y}}}$, $\mathrm{Q}_{\Phi_{\hat{G}}}$, and $\mathrm{Q}_{\Phi_{E}}$, and optimize the final classifier $\mathrm{P}_{\theta_Y}$. According to Proposition~\ref{prop:elbo_gy}, the objective in the M-step is to minimize the negative log-likelihood of $\mathrm{P}(Y|\hat{G}, E, G)$. The loss function can be formally represented by: 
\begin{align}
    \mathcal{L}_{\mathbf{M}} = \mathrm{CrossEntropy}(Y_T, \hat{Y}) ,
\end{align}
where $Y_T$ represents ground-truth graph labels. We iteratively perform E-step and M-step until model convergence. 

\textbf{[Convergence discussion.]}
We next discuss {the convergence of EM algorithm in \algo}. Our objective is to optimize Eq.~\ref{form:elbo}. In both E- and M-steps, the corresponding objectives are optimized towards a direction of decreasing values, which also decreases the value of Eq.~\ref{form:elbo}. Further, since KL divergence of latent variables and negate of reconstruction probability both have a clear lower bound of 0, Eq.~\ref{form:elbo} will also be lower bounded by 0. Therefore, with a clear lower bound and each step optimized towards smaller values, the EM algorithm guarantees convergence in our method. 

\textbf{[Time complexity analysis.]}
While \algo\ adopts an EM-based framework, 
its major time complexity comes from GNNs and MLPs implemented, which are both linear to the number of nodes in the graph.
Therefore, the overall time complexity of our method is still linear to the number of nodes in the graph.

\section{Experiments}
\label{sec:experiments}

\subsection{Setup}
\label{subsec:setups}
To make fair comparison between \algo\ and SOTA methods, experiments are conducted on molecular learning and NLP tasks in GOOD~\cite{good_22_gui} and DrugOOD~\cite{drugood_23_ji}. To further assess the model performance under various distributional shifts, this study presented results from both covariate and concept shift scenarios in the GOOD benchmark. However, since DrugOOD benchmark does not include concept shifts, only the results from covariate shift are reported. 

For fairness, we select representative general OOD baselines and graph-specific baselines. The former includes ERM, IRM~\cite{irm_20_arjovsky}, VREx~\cite{vrex_21_krueger}, 
and DANN~\cite{dann_16_ganin}, while the latter consists of DIR~\cite{dir_22_wu}, GSAT~\cite{gsat_22_miao}, CIGA~\cite{ciga_22_chen}, GALA~\cite{gala_23_chen}, 
AIA~\cite{aia_23_sui}, iMoLD~\cite{imold_23_zhuang}, 
EQuAD~\cite{equad_24_yao}, IGM~\cite{igm_24_jia},
and LECI~\cite{leci_23_gui}. Due to fundamental distinctions between methods that focus on graph-level tasks and node-level ones, baselines for node-level predictions like EAGLE~\cite{eagle_23_yuan}, CaNet~\cite{canet_24_wu}, and MARIO~\cite{mario_24_zhu} are not included.

Note that differing from previous studies that primarily conduct experiments under covariate shifts, this paper extensively investigates generalizing graphs under more challenging concept shifts on real-world data, thereby providing a thorough model evaluation. The definitions of the two shifts are given in Section~\ref{subsec:prob_form}. Further, hyperparameter sensitivity results are reported in Section~\ref{subsec:hyper_param}.

\subsection{Main Results}
\label{subsec:real_sce}
We conduct experiments on real-world benchmarks, which are molecule learning and text sentiment analysis.
Results are given in Table~\ref{tab:real-world_results}. From the table, we observe that:

\textbf{Molecule learning.} \algo-v2 surpasses other competitors on multiple benchmarks, while its simplified version, \algo-v1, also achieves commendable performance, ranking second in most benchmarks. 
These results support modeling environmental information jointly with graph rationales. Among baselines, LECI achieves the best performance in covariate shift scenarios. It obtains environment through encoding the input graph $G$ into graph-level embedding and imposes a stringent assumption on the independence among causal graphs, environments, and labels. Since covariate shift only happens on $\mathrm{P}(G)$ and affects less on graph labels $\mathrm{Y}$, the independence assumption between environments and labels improves the performance of LECI. However, concept shift happens on $\mathrm{P}(Y|G)$, which means the spuriousness of environment is weakened. As a result, LECI performs poorly under concept shift conditions due to the strict independence assumption on $\mathrm{Y}$ and $\mathrm{E}$. 
While IGM achieves the best results on LBAP-IC50-Assay and -Size benchmarks, it does not perform well on GOOD benchmarks.

\textbf{Text sentiment analysis.} Our methods \algo-v2 and \algo-v1 lead other baselines on all the tasks except GOODSST2 with covariate shift. In this task, as pointed out in~\cite{leci_23_gui}, due to the presence of significant amount of noise in graph labels and environment labels, hyperparameters in LECI that enforce the independence between environment and graph labels are set to very small values or even zero. This is consistent with Theorem~\ref{theo:dep_on_env}: environmental information can retain label-predictive signal.

\subsection{Ablation Study}
\label{subsec:abl_stu}
We next comprehensively validate the effectiveness of main components in \algo, including
training techniques, 
regularization terms 
and inferred latent variables.
Variants of \algo\ are categorized into three classes. The first class includes w/ OBI, w/o EM, and w/o GRL. These variants are implemented to validate different techniques used in the model framework.
Specifically,
the w/ OBI means \algo\ with the original Bayesian inference, 
where latent environments and graph rationales are optimized based on KL divergence.
The w/o EM indicates model parameters are updated simultaneously instead of the EM algorithm. 
The w/o GRL is \algo\ with no gradient reverse layer attached after the GNN encoder in the environment embedding generator. Further, the second class of variants consist of w/o $\mathcal{L}_{env}$, w/o $\mathcal{L}_{cl}$, and w/o $H(\tilde{Y})$, which are adopted to measure the importance of each term in the objective function of Eq.~\ref{eq:func_v2} on the overall model performance. To verify the role of the latent environment and graph rationale in model inference, we implement the third class of variants of \algo\ to evaluate the expressiveness of inferred variables. Variants include w/o $E$, w/o $\hat{G}$, and w/o $E \& \hat{G}$. They replace latent environments, graph rationales, or both with standard Gaussian noises, respectively. Note that all the experiments are based on \algo-v2.

\begin{table}[h]
    \centering
    \caption{Ablation study on \algo}
    \label{tab:ablation}
    \resizebox{\linewidth}{!}{
        \begin{tabular}{lcccc}
        \toprule
        Dataset & \multicolumn{2}{c}{HIV} & Twitter & SST2 \\
        \cmidrule(r){1-1} \cmidrule(l){2-5}
        Domain & Scaffold & Size & Length & Length \\
        \cmidrule(r){1-1} \cmidrule(lr){2-3} \cmidrule(lr){4-4} \cmidrule(l){5-5}
        Shift & Covariate & Concept & Covariate & Concept \\
        \midrule
        ERM & 72.01(1.02) & 61.88(4.27) & 57.77(0.75) & 71.35(0.56) \\
        LECI & 72.50(0.85) & 55.91(2.76) & 59.64(0.60) & 72.97(1.62) \\
        \algo-v1 & 75.51(0.66) & 76.33(1.99) & 60.15(0.58) & 73.58(0.56) \\
        \algo-v2 & \textbf{76.87(0.80)} & \textbf{80.65(8.35)} & \textbf{60.54(0.43)} & \textbf{75.48(1.17)} \\
        \midrule
        w/ OBI  & 73.02(1.89) & 62.67(7.21) & 56.37(1.41) & 71.26(1.04) \\
        w/o EM & 74.39(1.76) & 61.63(2.97) & 56.84(0.68) & 72.70(1.57) \\
        w/o GRL & 75.69(0.77) & 64.77(4.63) & 57.52(0.21) & 73.68(0.83) \\
        \midrule
        w/o $\mathcal{L}_{env}$ & 76.12(0.73) & 77.90(1.23) & 60.31(0.50) & 74.42(1.10) \\
        w/o $\mathcal{L}_{cl}$ & 75.99(0.96) & 77.78(2.40) & 60.25(0.52) & 73.80(1.36) \\
        w/o $H(\tilde{Y})$ & 74.10(1.48) & 63.86(3.19) & 59.83(0.08) & 72.63(0.96) \\
        \midrule
        w/o $E$ & 69.36(0.36) & 58.68(5.26) & 58.00(0.79) & 69.24(1.63) \\
        w/o $\hat{G}$ & 68.94(3.42) & 56.33(7.69) & 57.81(1.55) & 67.68(0.82) \\
        w/o $E \& \hat{G}$ & 61.08(2.34) & 54.48(1.53) & 56.42(1.07) & 64.41(2.19) \\
        \bottomrule
        \end{tabular}
        }
\end{table}

Table~\ref{tab:ablation} summarizes the experimental results.
From a global perspective, we observe that 
\algo\ outperforms all its variants. 
We next provide detailed analysis on these variants.

For the first class of variants with different training techniques,
w/ OBI achieves the worst performance. 
This verifies the rationality of adopting 
generalized Bayesian inference in the absence of reliable prior knowledge. The performance advantage of \algo\ over w/o EM shows the effectiveness of EM algorithm in optimizing variational inference problems. 
The w/o GRL variant also performs worse than \algo, which indicates the importance of gradient reverse layer in latent variable inference. 

For the second class, we see that
the environment alignment (w/o $\mathcal{L}_{env}$) has the minimal impact on overall performances, which leads to the smallest performance drop.
This further demonstrates the small reliance of model performance on the environment labels.
Hence,
{when environment labels are absent, 
we can remove the term $\mathcal{L}_{env}$ from our objective}.
Also, the performance drops of both w/o $\mathcal{L}_{cl}$ and w/o $H(\tilde{Y})$ indicate that the contrastive loss and the negative entropy of pseudo-labels are essential.
\begin{figure}[h]
    \centering
    \centerline{\includegraphics[width=\linewidth]{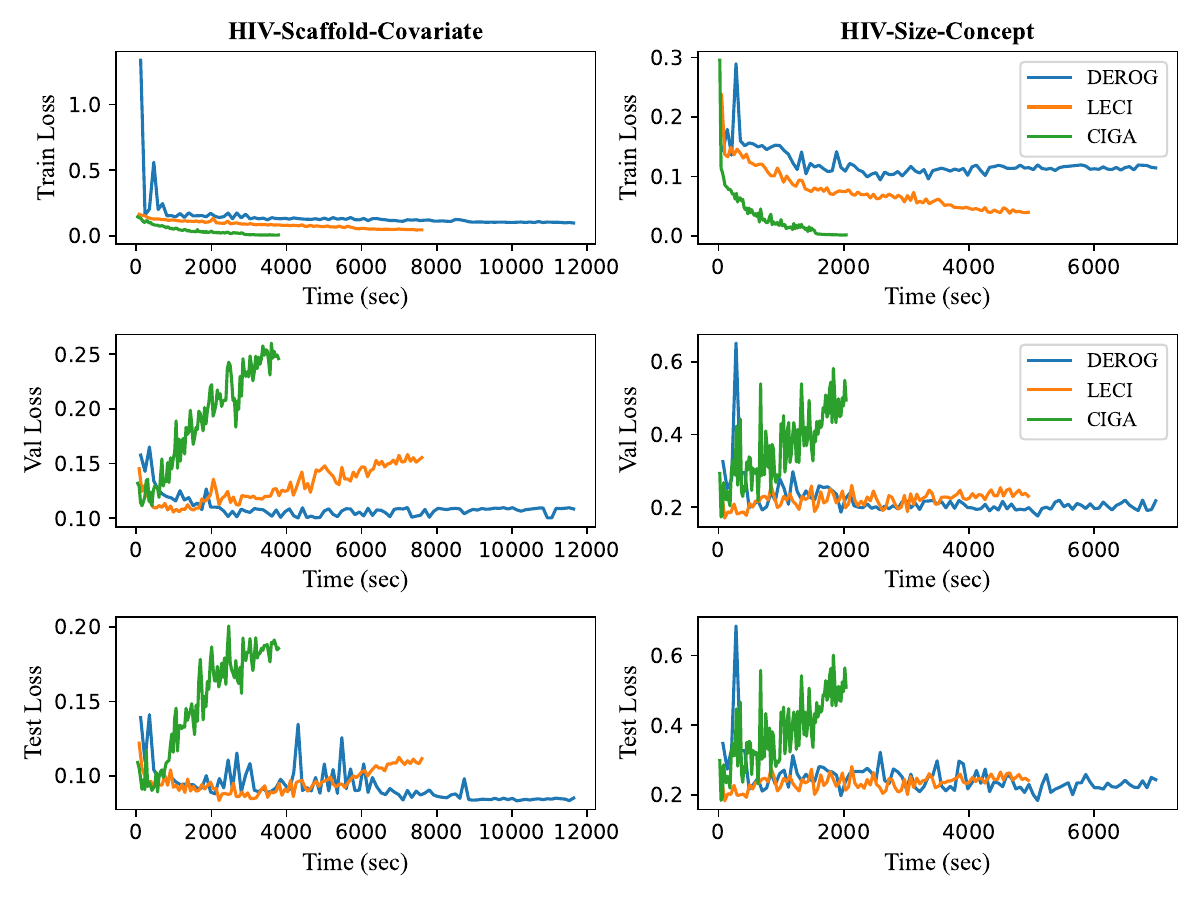}}
    \caption{Efficiency study}
    \Description{Line plots used for the efficiency study.}
    \label{fig:time_ana}
\end{figure}

\begin{figure*}[tb]
    \begin{center}
    \centerline{\includegraphics[width=0.8\linewidth]{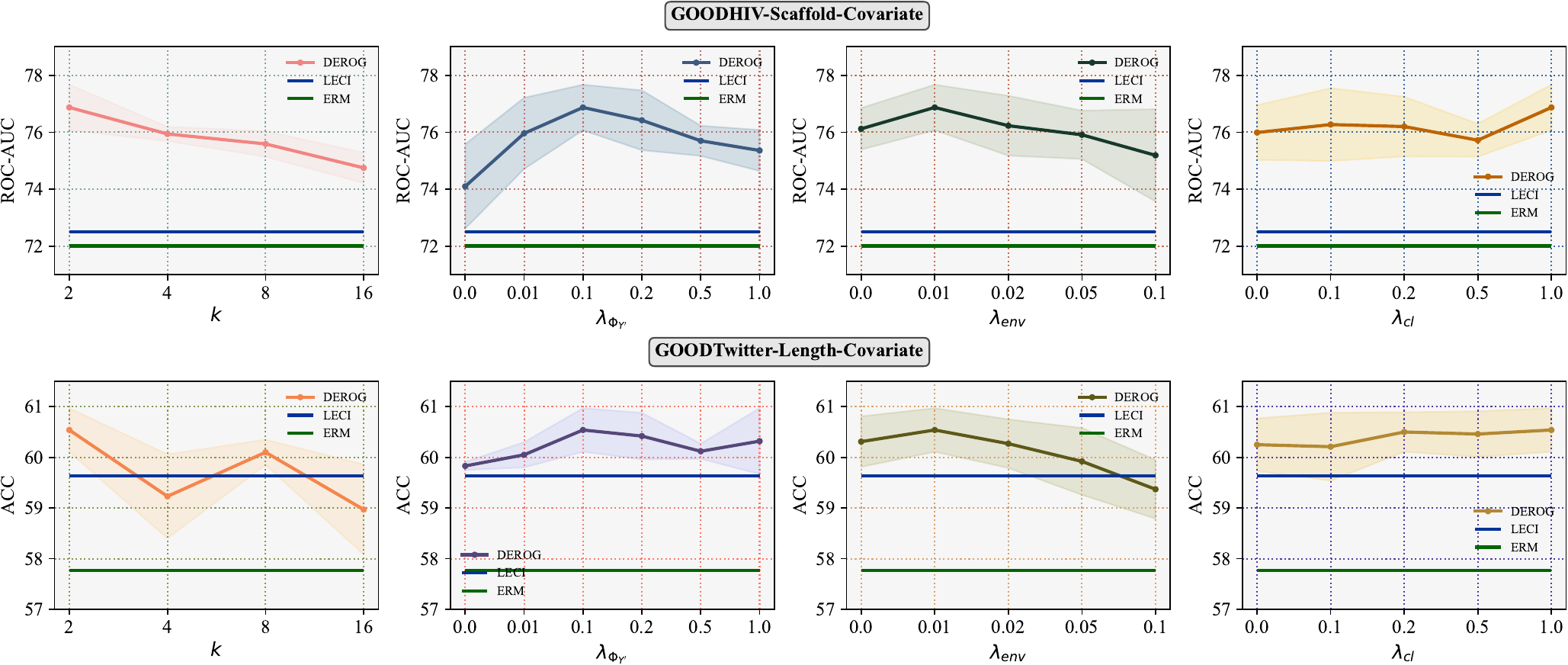}}
    \centerline{\includegraphics[width=0.8\linewidth]{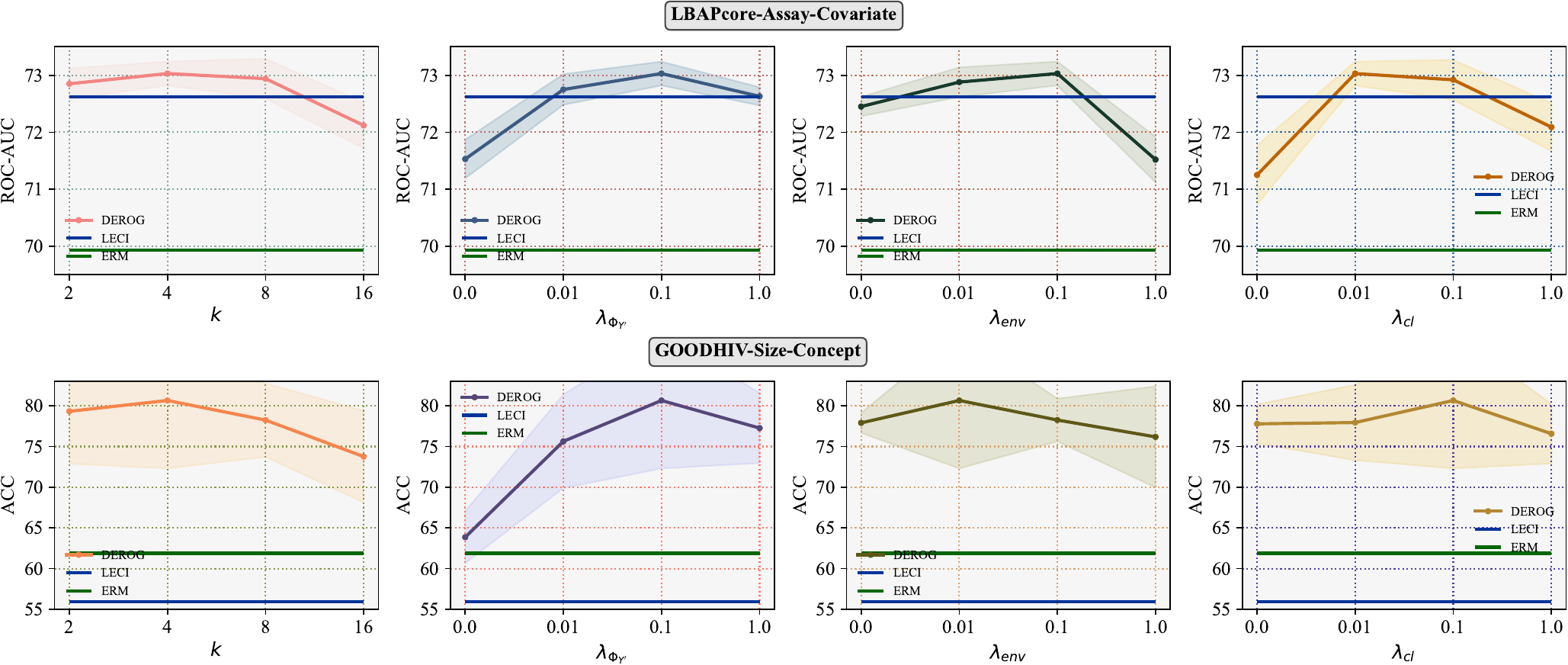}}
    \caption{Performance of \algo\ under various hyperparameter settings.}
    \Description{Two plots showing DEROG performance under different hyperparameter settings.}
    \label{fig:param_sen}
    \end{center}
\end{figure*}

For the third class, the replacement of inferred variables with random noises results in significant performance penalties. This shows the necessity of both environments and graph rationales. 
Among them, w/o $\hat{G}$ leads to a larger performance drop than w/o $E$, 
which indicates that graph rationales are of more important. However, environment is also unignorable.

\subsection{Efficiency Study}
\label{subsec:comp_stu}
To investigate the model efficiency, 
we show the training and validation loss of \algo\ over time on two benchmark datasets.
We further take LECI and CIGA as reference 
because the former is the runner-up method while the latter 
is simple and thus efficient.
The results are drawn in Fig.~\ref{fig:time_ana}. 
We observe that 
CIGA is most efficient, while \algo\ converges faster than LECI, in particular, on the HIV-Size-Concept case.
However,
both CIGA and LECI converge to a lower loss on \textbf{in-distribution training set}, but a higher and divergent loss on \textbf{OOD validation set},
while \algo\ stands in the opposite.
The stable convergence curve of \algo\ on the validation set 
further explains its effectiveness 
in graph OOD generalization;
also, it is efficient.

\subsection{Hyperparameter Sensitivity Analysis}
\label{subsec:hyper_param}
We further conduct a hyperparameter sensitivity analysis on \algo.
We evaluate the model performance w.r.t. varying values of $k$, $\lambda_{\Phi_{\tilde{Y}}}$, $\lambda_{env}$, and $\lambda_{cl}$ on GOODHIV, LBAPcore, and GOODTwitter.
Here, $k$ is the number of negative samples in contrastive loss; $\lambda_{\Phi_{\tilde{Y}}}$, $\lambda_{env}$, and $\lambda_{cl}$ represent the weighting factor of $H(\mathrm{Q}_{\Phi_{\tilde{Y}}}(\tilde{Y}|G))$ in Eq.~\ref{eq:entropy}, the terms $\mathcal{L}_{env}$ and $\mathcal{L}_{cl}$ in Eq.~\ref{eq:func_v2}, respectively. 
In our experiments, we vary one hyperparameter with others fixed.
For better comparison, two representative baselines LECI and ERM are included as reference. 
Specifically, LECI is the runner-up method while ERM is a widely adopted method for OOD problem. The results are shown in Fig.~\ref{fig:param_sen}.
Despite the existence of some fluctuations, our method still outperforms runner-up baselines on a wide range of hyperparameter values. These results demonstrate the insensitivity of our method towards hyperparameters.






\section{Conclusion}
\label{sec:conclusion}
This paper addresses the  graph OOD generalization problem. We first provided systematic analyses on the restricted applicability of existing invariant learning methods in complex real-world scenarios,
specifically regarding the underlying assumptions made for environments, invariant sub-graphs, and labels. 
After that, we re-evaluated the roles of environment and graph rationales and constructed a probabilistic graphical model \algo\ concerning the conditional distribution of labels. 
Experimental results under complicated real-world distribution shifts show that \algo\ 
can consistently provide superior performance.


\bibliographystyle{ACM-Reference-Format}
\bibliography{bibfile}

\end{document}